\documentclass[conference,letterpaper]{IEEEtran}

\usepackage{microtype}
\usepackage{graphicx}
\usepackage{subfigure}
\usepackage{booktabs} 
\usepackage{array,multirow}
\usepackage{enumitem}

\newcommand{\bunderline}[1]{\underline{#1\mkern-4mu}\mkern4mu }

\usepackage{algorithm}
\usepackage[noend]{algorithmic}

\newcounter{ALC@tempcntr}

\newcommand{\NEWSECTION}[1]{%
    \setcounter{ALC@tempcntr}{\arabic{ALC@rem}}
    \setcounter{ALC@rem}{1}
    \item {  \bf  #1 }  
    \setcounter{ALC@rem}{\arabic{ALC@tempcntr}}
}%

\usepackage{setspace}
\usepackage{amsthm,amsfonts}
\usepackage{bm}
\usepackage{amssymb}

\makeatletter
\newcommand*{\rom}[1]{\expandafter\@slowromancap\romannumeral #1@}
\makeatother

\newtheorem{theorem}{Theorem}
\newtheorem*{proof*}{proof}
\newtheorem{lemma}{Lemma}
\newtheorem{remark}{Remark}

\newtheorem{corollary}{Corollary}

\newtheorem{assumption}{Assumption}

\usepackage[utf8]{inputenc} 
\usepackage[T1]{fontenc}
\usepackage{url}
\usepackage{ifthen}
\usepackage{cite}
\usepackage[cmex10]{amsmath} 

\DeclareMathOperator{\bw}{\mathbf{w}}
\newcommand\inner[2]{\langle #1, #2 \rangle}

\begin{document}

\title{Energy-Harvesting Distributed Machine Learning}

\author{%
  \IEEEauthorblockN{Ba\c{s}ak~G{\"u}ler}
  \IEEEauthorblockA{University of California, Riverside\\
                    Riverside, California\\
                    \emph{bguler@ece.ucr.edu}}
  \and
  \IEEEauthorblockN{Aylin Yener}
  \IEEEauthorblockA{The Ohio State University\\ 
                    Columbus, Ohio\\
                    \emph{yener@ece.osu.edu}}
}


\maketitle



\begin{abstract}
This paper provides a first study of utilizing energy harvesting for sustainable machine learning in distributed networks. We consider a distributed learning setup in which a machine learning model is trained over a large number of devices that can harvest energy from the ambient environment, and  develop a practical learning framework with theoretical convergence guarantees. We demonstrate through numerical experiments that the proposed framework can significantly outperform energy-agnostic benchmarks. Our framework is scalable, requires only local estimation of the energy statistics, and can be applied to a wide range of distributed training settings, including machine learning in wireless networks, edge computing, and mobile internet of things.  
\end{abstract}

\section{Introduction}
The environmental impact of large-scale machine learning is a major challenge against the sustainability of future smart ecosystems. 
For instance, the carbon emission of training a single machine learning model can get as large as the lifetime of five cars \cite{strubell2019energy}. 
The environmental impact will be even greater with the emergence of machine learning in distributed environments, where millions of devices are expected to participate in training on a regular basis. 
This, combined with the fact that state-of-the-art machine learning models are trained over billions of parameters \cite{brown2020language}, calls for a novel design paradigm for large-scale machine learning. 

In this paper, we propose energy harvesting \cite{ulukus2015energy} for the design of sustainable distributed machine learning systems. 
We consider a distributed training scenario with $N$ clients (users), who wish to collaborate to train a machine learning model. Each user holds a local dataset $\mathcal{D}_i$, and the goal is to train a machine learning model over the joint dataset $\mathcal{D}_1, \ldots, \mathcal{D}_N$. Training is performed through distributed stochastic gradient descent (SGD)   coordinated through a central server, who maintains a global model. At each iteration of training, the server sends the current estimate of the model parameters to the users. Users then locally update the global model by computing a local gradient on their local dataset, and send their local updates to the server. The server then aggregates the local updates from the users,  updates the global model, and sends the updated model back to the users. 
Unlike the conventional distributed SGD setting, in this work, users receive energy through an energy harvesting process \cite{ulukus2015energy, radousky2012energy, tutuncuoglu2017binary, ozel2015fundamental, tutuncuoglu2012optimum, ozel2011transmission, varan2016delay, tutuncuoglu2015throughput, tutuncuoglu2015energy, tutuncuoglu2012sum, gurakan2013energy, yang2011optimal, ozel2012achieving}, and can only participate in training if they have energy available to do so.

Energy and resource efficiency in machine learning has been studied in various notable works \cite{zeng2020energy,  wang2019adaptive}. 
Broadly, these settings can be categorized into two. The first line of work focuses on minimizing the energy consumption of the compute or communication framework \cite{zeng2020energy}. The second line of work, on the other hand, is focused on minimizing the training loss within a given energy budget, where all of the energy is available at the beginning of training \cite{wang2019adaptive}. In contrast, our work focuses on training with devices that can harvest small amounts of energy from the ambient environment, where energy arrivals are intermittent and non-homogeneous across different devices. 

Prior to this work, user sampling for distributed machine learning has been primarily investigated in the context of improving communication efficiency or convergence rate \cite{DBLP:conf/nips/AgarwalSYKM18, LiHYWZ20, chen2020optimal, cho2020client, ren2020scheduling, sun2020energy, amiri2020update}. In these works, the primary goal is to either select a small set of users to participate at a given training iteration in order to  reduce the overall  communication overhead or due to bandwidth limitations, or to select a few informative users to maximize the convergence rate of training, with the assumption that all users are available to participate in training if selected. In contrast, in our setting, users can only participate in training if they have available energy. Moreover, the energy availability of different users can be  different. 
Several notable works have considered distributed learning when users have a chance to drop out, unlike the current setup, in these settings, user dropouts occur uniformly at random  \cite{mcmahan2017communication, bonawitz2017practical, so2020turbo}. 

We demonstrate that energy-harvesting can be a good candidate for machine learning in distributed networks, through a practical distributed training framework with theoretical convergence guarantees. 
Our experiments show that the proposed framework significantly outperforms the alternative distributed SGD benchmarks that are agnostic to the energy arrival process. 
We hope our work to open up new research directions in leveraging energy-harvesting for sustainable machine learning in large-scale mobile and edge networks.

\section{System Model}\label{sysmodel}
\subsection{Training Setup}\label{training}
We consider a distributed training setup in a network with $N$ devices (users). The users are connected through a central server who coordinates the training. User $i$ has a local dataset $\mathcal{D}_i$, consisting of $D_i$ data points. We define the total number of data points in the network as $D = \sum_{i\in[N]} D_i$. The goal is to train a model $\mathbf{w}$ that minimizes a global loss function
\begin{equation}
F(\mathbf{w}) = \frac{1}{D} \sum_{i=1}^{N} \sum_{j=1}^{D_i} l(\mathbf{w}, \mathbf{x}_{ij}) \label{eq:global1}
\end{equation} 
where $l(\mathbf{w},  \mathbf{x}_{ij})$ denotes the loss of data point $\mathbf{x}_{ij}$ from the local dataset of user $i$. 
Note that the loss function in \eqref{eq:global1} is evaluated with respect to the entire set of data points that belong to the $N$ users. 
As such, equation \eqref{eq:global1} can also be written as
\vspace{-0.12cm}\begin{equation}\label{global-loss2}
F(\mathbf{w}) = \sum_{i=1}^{N} p_i F_i(\mathbf{w})
\vspace{-0.12cm}\end{equation}
where $p_i = \frac{D_i}{D}$ such that $\sum_{i=1}^n p_i = 1$, and  
\vspace{-0.12cm}\begin{equation}\label{localoss}
F_i(\mathbf{w}) = \frac{1}{D_i} \sum_{j=1}^{D_i} l(\mathbf{w}, \mathbf{x}_{ij})
\vspace{-0.12cm}\end{equation}
represents the local loss function of user $i$. 

Training is performed through distributed SGD, in which the model parameters are updated iteratively in the negative direction of the gradient. 
Each iteration is represented by a discrete time instant $t\in\{0, 1, 2, \ldots\}$. 
The current estimation of the model parameters at iteration $t$ is represented by a $d$-dimensional vector $\mathbf{w}^{(t)}\in \mathbb{R}^d$, where $d$ is the model size. 

We now review the conventional distributed SGD protocol. 
In this setting, at the beginning of each iteration, the server sends $\mathbf{w}^{(t)}$ to the users. 
Then, user $i\in\{1, \ldots, N\}$ computes a local stochastic gradient, 
\vspace{-0.12cm}\begin{equation}\label{gradient}
g_i (\mathbf{w}^{(t)}, \xi_{i}^{(t)}) \triangleq \nabla F_i(\mathbf{w}^{(t)}, \xi_{i}^{(t)}) 
\vspace{-0.12cm}\end{equation} 
by using a (uniformly) random sample $\xi_{i}^{(t)}$ from the local dataset $\mathcal{D}_i$. 
Hence,  the stochastic gradient is an unbiased estimator of the true gradient of user $i$,  

\vspace{-0.12cm}\begin{equation}\label{unbiased}
\mathbb{E}_{\xi_i^{(t)}} [\nabla F_i(\mathbf{w}^{(t)}, \xi_{i}^{(t)})] 
= \nabla  F_i(\mathbf{w}^{(t)}),  
\vspace{-0.12cm}\end{equation}
where $\nabla  F_i(\mathbf{w}^{(t)})$ is the gradient of the local loss function in  \eqref{localoss}. 
The gradient of the global loss function in  \eqref{eq:global1} is given by,
\vspace{-0.12cm}\begin{equation}\label{global-gradient}
\nabla  F (\mathbf{w}^{(t)}) \triangleq \sum_{i=1}^N p_i \nabla F_i(\mathbf{w}^{(t)}). 
\vspace{-0.12cm}\end{equation} 
After the local computations, users send their local gradients from \eqref{gradient} to the server. The server then updates the model,
\vspace{-0.12cm}\begin{equation}\label{eq:updateconv}
\mathbf{w}^{(t+1)} = \mathbf{w}^{(t)} - \eta \sum_{i=1}^N p_i  g_i(\mathbf{w}^{(t)}, \xi_{i}^{(t)})
\vspace{-0.12cm}\end{equation}
where $\eta$ is the learning rate (step size), and sends the updated model back to the users for the next iteration.

\subsection{Energy Harvesting Profile of the Users} 
This work considers devices that are powered by the energy harvested from the ambient environment, such as RF, solar, or kinetic energy \cite{ulukus2015energy, radousky2012energy}. 
We assume that one step of the SGD protocol costs a unit amount of energy at each user, which includes computing the local gradient from \eqref{gradient} and sending it to the server. 
It is also assumed that each user has a unit battery that can store enough energy for one step SGD.

We let $E_i^t$ denote the energy arrival process at user $i$, in particular $E_i^t=1$ if user $i$ receives energy at time $t$ and $E_i^t=0$ otherwise. The specific distribution of the energy arrivals depends on the harvesting process.  
Our focus is on the following energy harvesting scenarios.

\vspace{0.2cm}
\noindent
\subsubsection{Deterministic Energy Arrivals}\label{sec:deter}

We first consider a deterministic energy harvesting scenario in which energy arrivals are known by each user in advance. 
We assume that energy may arrive at arbitrary non-overlapping time instances, and let 
$\mathcal{I}_i = \{t: E_i^t = 1\}$  
denote the set of time instances at which user $i$ receives energy. We also define $\bunderline{I}_i^t = \max_{t':  t'\leq t, \;\; t' \in \mathcal{I}_i} t'$   
for the time of the most recent energy arrival up to $t$, and $\bar{I}_i^t = \min_{t':  t'> t, \;\; t' \in \mathcal{I}_i} t'$ 
for the time of the next energy arrival after time $t$. 
Finally, for a given $t$, we define the duration between 
$\bunderline{I}_i^t$ and $\bar{I}_i^t$ as, 
\begin{equation}
T_i^t = \bar{I}_i^t  - \bunderline{I}_i^t 
\end{equation}

\subsubsection{Stochastic Energy Arrivals}\label{subsec:stoc}
We next consider the stochastic energy harvesting scenario where energy arrivals are modeled through a stochastic process. Unlike the deterministic setting, users do not know the exact time instant at which energy will be received, but only the probabilistic model governing the underlying harvesting process. 
Our focus is on the following stochastic arrival scenarios.

\vspace{0.2cm}
\noindent
{\it (Binary Arrivals)} 
In the binary energy arrival setup, at each time instant, user $i$ receives a unit amount energy with probability $\beta_i$. 
More specifically, we let $E_i^t\sim\text{Bern}(\beta_i)$:
\begin{equation}\label{binary}
E_i^t = \left \{ \begin{matrix} 1 & \text{ with probability } & \beta_i\\ 0 &  \text{ with probability } & \quad \; 1-\beta_i \end{matrix} \right .
\end{equation}
where $\beta_i\in(0,1]$, to represent whether or not user $i$ receives energy at time $t$. 
Parameter $\beta_i$ quantifies how frequent user $i$ receives energy, and may vary from one user to another.

\vspace{0.2cm}
\noindent
{\it (Uniform Arrivals)} 
We next consider a uniform energy arrival scenario in which device $i$ receives a unit amount of energy at a uniformly random time instant every $T_i$ time instants. 
Formally, for any $t$ such that $t\mod T_i = 0$, user $i$ receives a unit amount of energy at a uniformly random time instant within $\{t, \ldots,  t + T_i-1\}$.

Note that this is not an immediate generalization of the first setting, as in the former setup there is a non-zero probability that user $i$ will never receive energy in $T_i$ time instants. In contrast, in the second setting, user $i$ receives a unit amount energy with probability $1$ at every $T_i$ time instants, but the exact time instant at which energy is received is unknown.

As we demonstrate in our experiments, the conventional  distributed SGD strategy from Section~\ref{sysmodel} might bias the model towards users that have more frequent energy arrivals, causing a performance loss in training. 
As such, the training strategy should take into account the energy arrival patterns of the users. 

\vspace{0.2cm}
\noindent
{\bf Main Problem. } 
Given the above training and energy harvesting settings, the main problem we study in our work is, \emph{``How to design a distributed stochastic gradient descent framework for energy harvesting devices, where energy arrivals are intermittent and heterogeneous, while ensuring theoretical convergence guarantees?''.}

In the sequel, we provide a simple energy harvesting distributed learning strategy with provable convergence guarantees. The proposed strategy takes into account the intermittent energy availability due to the energy harvesting process of the individual users while ensuring that the model does not bias towards any particular user.

\section{Energy Harvesting Distributed SGD} 

\subsection{Distributed SGD with Deterministic Energy Arrivals}

We first study the deterministic energy harvesting scenario and provide a simple distributed training framework with theoretical convergence guarantees. The individual steps of our framework is provided in Algorithm~\ref{Alg-deterministic}. 
Our framework consists of three main components, user scheduling, local gradient computations, and server-side model update.

\setlength{\textfloatsep}{10pt}

\begin{algorithm}[t]
\small
  \caption{Distributed SGD with Deterministic Energy Arrivals}\label{Alg-deterministic} 
  \begin{algorithmic}[1]
    \INPUT{Number of devices $N$, local dataset $\mathcal{D}_i$ of device $i\in[N]$, number of iterations $T$, initial model parameters $\mathbf{w}^{(0)}$. 
    } \
    \OUTPUT{Model parameters (weights) $\mathbf{w}^{(T)}$.} \ 
    \vspace{0.1cm}
    
    \FOR{user $i=1,\ldots,N$}
    \STATE Initialize $U_i^t =   0$ for $t\in[T]$.
    \ENDFOR

    \vspace{0.1cm}

    \FOR{iteration $t=0,\ldots,T-1$} 

     \vspace{0.1cm}
    \NEWSECTION{Users $i=1, \ldots, N:$}
    \vspace{0.1cm}
    \IF{$ E_i^t = 1$} 
    \STATE Sample an integer $J$ uniformly random from $\{0, \ldots, T_i^t-1\}$.    
    \STATE Update $U_i^{t+J}  = 1$. 
    \ENDIF

    \IF{$ U_i^t = 1$} 
    \STATE Compute the local gradient $g_i(\mathbf{w}^{(t)}, \xi_{i}^{(t)})$. 
    \STATE Send $T_i^tg_i(\mathbf{w}^{(t)}, \xi_{i}^{(t)})$ to the server. 
    \ENDIF

    \vspace{0.1cm}
    \NEWSECTION{Server:}  
    \STATE Update the model according to \eqref{global-deter}.
    \STATE Send the model parameters $\mathbf{w}^{(t+1)}$ to the users.

    \ENDFOR

  \end{algorithmic}
\end{algorithm}

\subsubsection{User scheduling} \label{scheduling-deter}

The first component of our framework is user scheduling for training. Conventional user selection algorithms for distributed SGD are designed under the assumption that all users are inherently available to participate in the training process if selected, and employ a user sampling strategy to reduce the communication load or aim at selecting the users that will maximize the  convergence rate for training \cite{DBLP:conf/nips/AgarwalSYKM18, LiHYWZ20, chen2020optimal, cho2020client}. In contrast, in our setup, not all users can participate in the training process at all rounds. This is due to the intermittent energy arrivals, if a user has no energy at a given time instant, they will not be able to participate in training. 

A naive approach would be to utilize the conventional distributed SGD algorithm from \eqref{eq:updateconv}. 
However, doing so may bias the trained model towards users who have more frequent energy availability.   
Another approach is to wait until all users become available, and then use the conventional distributed SGD algorithm from \eqref{eq:updateconv}. However, waiting for all users to have enough energy can significantly increase the total training time needed to achieve a target performance level. 

Instead, we propose a practical scheduling strategy that can be performed locally by the users, while ensuring that the model does not bias towards any user. 
In this setting, whenever a user receives energy, i.e., $E_i^t=1$ for some $t$, the user samples an integer $J$ uniformly at random from the set $\{0, \ldots, T_i^t-1\}$, and participates at iteration $t+J$.

\subsubsection{Local gradient computation}

At the beginning of each training iteration, the server sends the current estimate of the model parameters $\mathbf{w}^{(t)}$ to the users. 
If a user decides to participate in the current training iteration $t$, according to the scheduling strategy from  Section~\ref{scheduling-deter}, it computes the local gradient from \eqref{gradient}. 
Then, the user sends to the server a scaled version of their local gradient,
\begin{equation}\label{weighted}
 T_i^t g_i(\mathbf{w}^{(t)}, \xi_{i}^{(t)}) = T_i^t \nabla F_i(\mathbf{w}^{(t)}, \xi_{i}^{(t)}) 
\end{equation}

\subsubsection{Server-side model update}
After receiving the local computations from \eqref{weighted} from the participating users, the server updates the model as:
\begin{equation}\label{global-deter}
\mathbf{w}^{(t+1)} = \mathbf{w}^{(t)} - \eta \sum_{i\in S_t} p_i \left ( T_i^t g_i(\mathbf{w}^{(t)}, \xi_{i}^{(t)})\right ) 
\end{equation}
where $\mathcal{S}_t$ denotes the set of users who have participated at round $t$. 
Note that due to the stochastic nature of the user scheduling process, $\mathcal{S}_t$ is random.

As we demonstrate in Section~\ref{04_Convergence}, this process provides theoretical convergence guarantees for the model. Moreover, the user scheduling process does not require a central coordinator and can be performed locally by the users, solely based on local energy  estimations, hence is scalable to large networks.

\begin{algorithm}[t]
\small
  \caption{Distributed SGD with Stochastic Energy Arrivals}\label{Alg-binary} 
  \begin{algorithmic}[1] 
    \INPUT{Number of devices $N$, local dataset $\mathcal{D}_i$ of device $i\in[N]$, number of iterations $T$, initial model parameters $\mathbf{w}^{(0)}$. 
    } \
    \OUTPUT{Model parameters (weights) $\mathbf{w}^{(T)}$.} \

    \vspace{0.1cm}

    \FOR{iteration $t=0,\ldots,T-1$}

     \vspace{0.1cm}
    \NEWSECTION{Users $i=1, \ldots, N:$}
    \vspace{0.1cm} 
    \IF{$ E_i^t = 1$} 
    \STATE Compute the local gradient $g_i(\mathbf{w}^{(t)}, \xi_{i}^{(t)})$. 
    \STATE Send $\gamma_i^tg_i(\mathbf{w}^{(t)}, \xi_{i}^{(t)})$ to the server. 
    \ENDIF

    \vspace{0.1cm}
    \NEWSECTION{Server:} 
    \STATE Update the model according to \eqref{global-binary}. 
    \STATE Send the model parameters $\mathbf{w}^{(t+1)}$ to the users.

    \ENDFOR

  \end{algorithmic}
\end{algorithm}

\subsection{Distributed SGD with Stochastic Energy Arrivals}

We next consider distributed training under the  stochastic energy harvesting setting. 
The training strategy again consists of three main components, user scheduling, local gradient computation, and server-side model update. 
We employ a \emph{best-effort} user scheduling strategy, where each user participates in training as soon as they receive energy, by computing the local gradient from \eqref{gradient}, and sending to the server a scaled gradient $\gamma_i^t g_i(\mathbf{w}^{(t)}, \xi_{i}^{(t)})$, where $\gamma_i^t = \frac{1}{\beta_i}$ and $\gamma_i^t = T_i$ for the binary and uniform energy arrival settings, respectively. 

After receiving the local computations from the participating users, the server updates the model as,
\begin{equation}\label{global-binary}
\mathbf{w}^{(t+1)} = \mathbf{w}^{(t)} - \eta \sum_{i\in S_t} p_i \left (\gamma_i^t g_i(\mathbf{w}^{(t)}, \xi_{i}^{(t)})\right ) 
\end{equation} 
The individual steps of this process are provided in Algorithm~\ref{Alg-binary}.

\label{section:stoc}

\section{Convergence Analysis} \label{04_Convergence}

We now state the convergence guarantees of our framework, by first reviewing a few common technical assumptions \cite{LiHYWZ20, stich2018local} that will be needed for our  convergence analysis. 

\begin{assumption}(Bounded variance) The variance of the stochastic gradients from \eqref{gradient} are bounded:
\begin{equation}\label{eq:bounded-var}
E_{\xi_i^{(t)}}[||g_i (\mathbf{w}^{(t)}, \xi_{i}^{(t)}) \!-\!  \nabla F_i(\mathbf{w}^{(t)})||^2]\leq \sigma^2 \; \text{ for } i\in [N]
\end{equation}
\end{assumption}

\begin{assumption}(Second moment bound) The expected squared norm of the stochastic gradients from \eqref{gradient} are bounded:
\vspace{-0.05cm}\begin{equation}\label{eq:bounded-norm}
E_{\xi_i^{(t)}}[||g_i (\mathbf{w}^{(t)}, \xi_{i}^{(t)})||^2]\leq G^2 \quad \text{ for } i\in [N]
\vspace{-0.05cm}\end{equation}
\end{assumption}

We also assume that the local loss functions $F_i(\mathbf{w})$ for $i\in[N]$ (and thus the global loss function $F(\mathbf{w})$) are $\mu$-strongly convex and $L$-smooth, as in \cite[Assumptions 1 and 2]{LiHYWZ20}. 
Next, we provide a key technical lemma.

\begin{lemma}\label{lemma-deter} (Unbiasedness)
For distributed SGD with deterministic energy arrivals,
\vspace{-0.1cm}\begin{equation}\label{eq:unbiased}
\mathbb{E}_{S_t}\left[\sum_{i\in S_t} p_i T_i^t g_i(\mathbf{w}^{(t)}, \xi_{i}^{(t)})\right] = 
\sum_{i=1}^N p_i g_i(\mathbf{w}^{(t)}, \xi_{i}^{(t)}),
\vspace{-0.1cm}\end{equation}
hence the user scheduling scheme is unbiased. Moreover, for distributed SGD with stochastic energy arrivals, the unbiasedness condition from \eqref{eq:unbiased} holds by replacing $T_i^t$ with $\frac{1}{\beta_i}$ and  $T_i$ for binary and uniform arrivals, respectively. 
\end{lemma}

\begin{proof}\label{proof-lemma1} 
We first define a Bernoulli random variable $\alpha_i^{t}$ to represent whether or not user $i$ participates at iteration $t$: 
\vspace{-0.1cm}\begin{equation}\label{alpha}
\alpha_i^{t} = \left \{ \begin{matrix} 1 & \text{ if user $i$ participates at time $t$}\\ 0 & \text{ otherwise}\end{matrix} \right .
\vspace{-0.1cm}\end{equation}
Then, for any given $t$,
\vspace{-0.1cm}\begin{equation}\label{eq:prob}
P[\alpha_i^{t} = 1]  = P[ J =  t - \bunderline{I}_i^t] = \frac{1}{T_i^t}
\vspace{-0.1cm}\end{equation}
By letting $\alpha_t \triangleq (\alpha_1^t,\ldots, \alpha_N^t)$, we find that,
\vspace{-0.1cm}\begin{align}
\!\!\mathbb{E}_{S_t}\!\left[\sum_{i\in S_t} p_i T_i^t g_i(\mathbf{w}^{(t)}, \xi_{i}^{(t)})\right] 
&\!=\! \mathbb{E}_{\alpha_t }\!\!\left[\sum_{i=1}^N \alpha_i^{t}  p_i T_i^t g_i(\mathbf{w}^{(t)}, \xi_{i}^{(t)}) \right] \!\label{lawof1} \\
&\!=\! \sum_{i=1}^N p_i T_i^t \frac{1}{T_i^t} g_i(\mathbf{w}^{(t)}, \xi_{i}^{(t)}) \!\label{iterated2}
\end{align}
where \eqref{lawof1} follows from $S_t = \sum_{i=1}^N \alpha_i^{t}$, and \eqref{iterated2} is from \eqref{eq:prob}.

The proof for stochastic arrivals follows the same lines along with the observation that, for the best-effort user scheduling strategy $P[\alpha_i^t =1]= P[E_i^t=1]$. \vspace{-0.12cm} \end{proof} 
We now state our convergence guarantees. 
\vspace{-0.12cm}\begin{theorem}\label{thm1-deter} 
For training a machine learning model from \eqref{eq:global1}, using the distributed SGD algorithm with deterministic energy arrivals and a constant learning rate $\eta \leq  \min \left\{ \frac{1}{2\mu}, \frac{1}{L} \right\}$. 
\vspace{-0.12cm}\begin{align}\label{eq:conv}
&\mathbb{E} [F(\mathbf{w}^{(T)})] - F(\mathbf{w}^*) \notag \\
&\hspace{0.5cm} \leq  \frac{L}{\mu} (1-\eta\mu)^T (F(\mathbf{w}^{(0)})-F(\mathbf{w}^*) - \frac{\eta C}{2}) 
+ \frac{\eta L C}{2\mu}  
\end{align}
in $T$ iterations, where $\mathbf{w}^{*}$ denotes the optimal model parameters that minimize the global loss function in \eqref{eq:global1}, and
\vspace{-0.12cm}\begin{equation}\label{eq:C}
C \triangleq \Big (\sum_{i=1}^N \Big ( T_{i, max}-1 \Big ) p_i^2 +  
\sum_{i=1}^N \sum_{j=1}^N p_i p_j \Big ) G^2,  
\vspace{-0.12cm}\end{equation} 
where $T_{i, {max}} \triangleq \max\{T_i^1, \ldots, T_i^T\}$ for $i=1\ldots, N$. 
\end{theorem}
\vspace{-0.3cm}
\begin{remark}
The first term in the right hand side of \eqref{eq:conv} vanishes as $T\rightarrow \infty$, whereas the second term $\frac{\eta L C}{2\mu}$ represents a non-vanishing error term due to the constant learning rate. By using a decreasing learning rate as in \cite{cho2020client, LiHYWZ20}, this term can also be made vanishing as $T\rightarrow \infty$. 
\end{remark}
\vspace{-0.4cm}
\begin{proof} (Sketch) 
The proof follows standard steps for the convergence analysis of distributed SGD algorithms \cite{DBLP:conf/nips/AgarwalSYKM18, LiHYWZ20, stich2018local}, hence we provide a proof sketch in the sequel. 
By letting $g_i^t \triangleq g_i(\bw^{(t)}, \xi_t)$, $\bw^* \triangleq  \arg\min_{\bw} F(\bw)$, and  $\xi_t\triangleq (\xi_1^{(t)}, \ldots, \xi_N^{(t)})$, from \eqref{global-deter} we find that,
\begin{align}
& \mathbb{E}_{S_t, \xi_t}[\lVert \bw^{(t+1)}   - \mathbf{w}^*\rVert^2] =  \mathbb{E}_{S_t, \xi_t}[ \lVert \bw^{(t)} - \mathbf{w}^*\rVert^2 ]  \notag \\ 
& \!-\! 2 \eta \mathbb{E}_{S_t, \xi_t}[\inner{\bw^{(t)} \!\! - \! \bw^*} {\!\!\sum_{i\in\mathcal{S}_t} p_i T_i^t g_i^t}]  
\!+\! \eta^2 \mathbb{E}_{S_t, \xi_t}[\lVert \sum_{i\in S_t} p_i T_i^t g_i^t\rVert^2] \label{eq:322}
\end{align} 

\vspace{-0.25cm}
\noindent
From Lemma~\ref{lemma-deter}, \eqref{unbiased}, and $\mu$-strong convexity, we observe that, 
\vspace{-0.1cm}\begin{align}
& \mathbb{E}_{S_t, \xi_t}[\inner{\bw^{(t)} - \bw^*} {\sum_{i\in\mathcal{S}_t} p_i T_i^t g_i^t}] \notag \\ 
& = \mathbb{E}_{S_t, \xi_t}[\inner{\bw^{(t)} - \bw^*} {\sum_{i\in\mathcal{S}_t} p_i T_i^t g_i^t - \sum_{i=1}^N p_i \nabla F_i (\bw^{(t)})} ]  \notag \\
&\hspace{1cm}+ \mathbb{E}_{S_t, \xi_t}[\inner{\bw^{(t)}-\bw^*}{\sum_{i=1}^N p_i \nabla F_i (\bw^{(t)})} ]  \label{eq:36} \\
& = \inner{\bw^{(t)} - \bw^*}{\nabla F(\bw^{(t)})}\label{eq:37}\\
& \geq    F(\bw^{(t)}) - F(\bw^*)  +   \frac{\mu}{2} \lVert \bw^*  -  \bw^{(t)} \rVert^2\label{eq:convex}
\end{align} 
We also have from Lemma~\ref{lemma-deter} that,
\vspace{-0.2cm}\begin{align}
 \mathbb{E}_{S_t, \xi_t}[\lVert \sum_{i\in S_t} p_i T_i^t g_i^t\rVert^2] & = \mathbb{E}_{S_t, \xi_t}[\lVert \sum_{i\in S_t} p_i T_i^t g_i^t 
- \sum_{i=1}^N p_i g_i^t 
\rVert^2]  \notag \\ 
& \quad \quad +\mathbb{E}_{S_t, \xi_t}[\lVert \sum_{i=1}^N p_i g_i^t 
\rVert^2] \label{eq:42}
\end{align} 
By combining \eqref{eq:322}, \eqref{eq:convex}, and  \eqref{eq:42}, we find that,
\vspace{-0.1cm}\begin{align}
& \mathbb{E}_{S_t, \xi_t}[\lVert \bw^{(t+1)}   \!-\! \mathbf{w}^*\rVert^2] \notag \\
& \!\leq\! (1\!-\!\eta\mu) \mathbb{E}_{S_t, \xi_t} [ \lVert \bw^{(t)} - \bw^*\rVert^2 ]
- 2 \eta (F(\bw^{(t)})- F(\bw^*)) \notag \\
& 
\!+\! \eta^2 \mathbb{E}_{S_t, \xi_t}[\lVert 
\sum_{i\in S_t} p_i T_i^t g_i^t \!-\! \sum_{i=1}^N p_i g_i^t 
\rVert^2] \!+\! \eta^2 \mathbb{E}_{S_t, \xi_t}[\lVert 
\sum_{i=1}^N p_i g_i^t \rVert^2] \label{eq:477}
\end{align}

\vspace{-0.5cm}
\noindent
By defining $\alpha_i^t$ as in \eqref{alpha} and $\alpha_t = (\alpha_1^t, \ldots, \alpha_N^t)$, 
\begin{align}
& \mathbb{E}_{S_t, \xi_t}[\lVert 
\sum_{i\in S_t} p_i T_i^t g_i^t - \sum_{i=1}^N p_i g_i^t \rVert^2] \notag \\
& = \mathbb{E}_{\alpha_t, \xi_t}[\lVert 
\sum_{i=1}^N p_i  ( \alpha_i^t T_i^t g_i^t -  g_i^t  ) \rVert^2] \\
& = \sum_{i=1}^N p_i^2 \mathbb{E}_{\alpha_t, \xi_t}[\lVert 
  \alpha_i^t T_i^t g_i^t -  g_i^t  \rVert^2] \notag \\
&+ \sum_{i=1}^N \sum_{\substack{j=1\\j\neq i}}^N   
\mathbb{E}_{\alpha_t, \xi_t}[
 \inner{p_i (\alpha_i^t T_i^t g_i^t -  g_i^t )}{p_j (\alpha_j^t T_j^t g_j^t -  g_j^t ) } ] \label{eq:48b} \\
& = \sum_{i=1}^N p_i^2 \mathbb{E}_{\alpha_t, \xi_t}[\lVert 
  \alpha_i^t T_i^t g_i^t -  g_i^t  \rVert^2] \label{eq:48c}\\ 
& = \sum_{i=1}^N p_i^2 (T_i^t)^2 \mathbb{E}_{\xi_t} [  \mathbb{E}_{\alpha_t| \xi_t}[
  (\alpha_i^t - \frac{1}{T_i^t})^2 \lVert g_i^t  \rVert^2|\xi_t]] \\ 
& \leq \sum_{i=1}^N p_i^2(T_{i, max}-1) 
   G^2 \label{eq:leqG}   
\end{align}
where \eqref{eq:48c} holds from \eqref{eq:prob} and that $(\alpha_i^t, g_i^t)$ is independent from  $(\alpha_j^t, g_j^t)$ for all $i\neq j$;  \eqref{eq:leqG} is from \eqref{eq:prob} and  \eqref{eq:bounded-norm}.  
Finally, 
\vspace{-0.2cm}\begin{align}
&\eta^2 \mathbb{E}_{S_t, \xi_t}[\lVert 
\sum_{i=1}^N p_i g_i^t \rVert^2] \notag \\ 
& \leq \sum_{i=1}^N p_i^2  \mathbb{E}_{\xi_t} [\lVert  g_i^t \rVert^2] 
+ \sum_{i=1}^N \sum_{\substack{j=1\\ j\neq i}}^N p_i p_j  \mathbb{E}_{\xi_t} [\lVert g_i^t\rVert \lVert g_j^t\rVert]   \label{eq:CS} \\
& \leq \sum_{i=1}^N p_i^2  \mathbb{E}_{\xi_t} [\lVert  g_i^t \rVert^2] 
+ \sum_{i=1}^N \sum_{\substack{j=1\\ j\neq i}}^N  \frac{p_i p_j }{2}\mathbb{E}_{\xi_t} [\lVert g_i^t\rVert^2 + \lVert g_j^t\rVert^2]    \label{eq:AM-GM} \\ 
& \leq \sum_{i=1}^N \sum_{j=1}^N  p_i p_j G^2   \label{eq:G2}
\end{align}
where \eqref{eq:CS} is from the Cauchy-Schwarz inequality; \eqref{eq:AM-GM} is from the AM-GM inequality;  \eqref{eq:G2} is from \eqref{eq:bounded-norm}.  
By combining \eqref{eq:477} and \eqref{eq:leqG} with \eqref{eq:G2} and noting that $- 2 \eta (F(\bw^{(t)})- F(\bw^*)) \leq 0$,   
\begin{align}
& \mathbb{E}_{S_t, \xi_t}[\lVert \bw^{(t+1)}   \!- \mathbf{w}^*\rVert^2] 
 \!\leq \! (1\!-\!\eta\mu) \mathbb{E}_{S_t, \xi_t} [ \lVert \bw^{(t)} - \bw^*\rVert^2 ]
 \notag \\
& \hspace{1cm} \!+\! \eta^2 \Big (\sum_{i=1}^N \Big ( (T_{i, max} - 1) \Big ) p_i^2 +  
\sum_{i=1}^N \sum_{j=1}^N p_i p_j \Big ) G^2 \!\!\!\! \label{eq:63}
\end{align}
The remainder of the proof follows from standard induction arguments as in \cite{cho2020client, LiHYWZ20}, hence is omitted. 
\end{proof}

\begin{corollary}
For distributed SGD with stochastic energy arrivals, Theorem~\ref{thm1-deter} holds by replacing $T_{i, max}$ with $\frac{1}{\beta_i}$ for binary arrivals and with $T_i$ for uniform arrivals, respectively. The convergence analysis follows the same steps.  
\end{corollary}

\section{Experiments}  \label{05_Experiments}

\begin{figure}[t]
\centering
\includegraphics[width=0.95\linewidth]{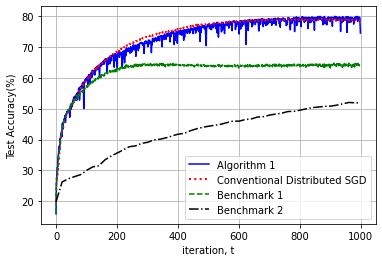}
\vspace{-0.3cm}\caption{Test accuracy of Algorithm~\ref{Alg-deterministic} compared to the benchmark distributed SGD algorithms for $N=40$ users on the CIFAR-10 dataset. }
\label{fig:test-accuracy}
\vspace{-0.2cm}
\end{figure}

In our experiments, we consider a conventional image classification task with $10$ classes on the CIFAR-10 dataset \cite{krizhevsky2009learning}, distributed over 40 users uniformly at random. Training is performed via distributed SGD using the convolutional neural network architecture from  \cite{mcmahan2017communication} (about $10^6$ model parameters).  
To demonstrate the impact of non-homogeneous energy-arrivals, users are partitioned into $4$ equal-sized groups $\mathcal{A}_0, \ldots, \mathcal{A}_3$ such that $\mathcal{A}_ k =\{i: i\mod 4 = k\}$, 
and the energy profiles of users in group $\mathcal{A}_k$ are set as:
\vspace{-0.12cm}\begin{equation}
E_i^t = \left \{ \begin{matrix} 
1 & \forall t  \text{ such that } t\!\!\!\mod \tau_k = 0 \\ 
0 &  \text{ otherwise } \\ 
\end{matrix} \right .
\end{equation}
for $i\in\mathcal{A}_k$, where $(\tau_0, \tau_1, \tau_2, \tau_3) = (1, 5, 10, 20)$. 
Therefore, users in group $\mathcal{A}_0$ receive energy at every time-instant $t$, whereas users in groups $\mathcal{A}_1$, $\mathcal{A}_2$, and $\mathcal{A}_3$ receive energy at every $5$, $10$, and $20$ time-instants, respectively. 
We compare our framework with the following distributed SGD benchmarks:

\noindent
{\bf Benchmark 1.} We first implement the distributed SGD framework from Section~\ref{sysmodel} when users participate in training as soon as they have energy available, by computing the gradient from \eqref{gradient} and sending it to the server, and then wait for the next energy arrival. Note that in this setting users do not scale the gradients with respect to the energy arrivals.

\noindent
{\bf Benchmark 2.} We then consider the distributed SGD framework from Section~\ref{sysmodel} when the global model is updated only if all users have enough energy to participate in training. That is, the server waits until all users have energy, then sends the current model parameters to the users, users compute the stochastic gradient from \eqref{gradient} and send it back to the server, and then the server updates the model as in \eqref{eq:updateconv}. 
Hence, in this case, the model is updated once every $t=20$ iterations.

Finally, we also implement the conventional distributed SGD framework from Section~\ref{sysmodel} when all users are available at every iteration, which represents our target (desired) accuracy level.  
We demonstrate our results in terms of the test accuracy with respect to time $t$ in Figure~\ref{fig:test-accuracy}. Our results show that Algorithm~\ref{Alg-deterministic} achieves the same accuracy level (about $80\%$) as conventional distributed SGD, whereas the two benchmarks achieve an accuracy of $64\%$ and $52\%$, respectively, within $t=1000$ iterations. This is due to the fact that the first benchmark favors users with more frequent energy arrivals, hence the model is biased. The second benchmark waits for all users to have enough energy before making a single SGD update, hence, even though the training algorithm is unbiased, its convergence rate is very slow. 
In contrast, Algorithm~\ref{Alg-deterministic} converges fast while achieving good accuracy.

\section{Conclusion}\label{conclusion} 
We have studied distributed machine learning when users have intermittent energy availability, and demonstrated a simple distributed learning strategy with provable convergence guarantees. Future directions include exploring optimal scheduling and training strategies with energy accumulation. We hope our study to open up further research on energy harvesting for sustainable learning in distributed and mobile networks.

\bibliographystyle{IEEEtran}
\bibliography{ref}

\IEEEtriggeratref{4}

\end{document}